\documentclass[conference]{IEEEtran}
\IEEEoverridecommandlockouts
\usepackage{graphicx}
\usepackage{epstopdf}
\usepackage{epsfig}
\usepackage{xcolor}
\usepackage{float}
\usepackage{color,soul}
\usepackage{amsmath}
\usepackage{amsthm}
\usepackage{amssymb}
\usepackage{bbm}
\usepackage{array}
\usepackage{algorithm,algpseudocode}
\newtheorem{theorem}{Theorem}[section]
\newtheorem{lemma}[theorem]{Lemma}

\newcommand{\myb}[1]{\boldsymbol{#1}}
\newtheorem{prop}{Proposition}

\usepackage{cite}
\usepackage{amsmath,amssymb,amsfonts}
\usepackage{graphicx}
\usepackage{textcomp}
\usepackage{xcolor}
\usepackage{breqn}
\def\BibTeX{{\rm B\kern-.05em{\sc i\kern-.025em b}\kern-.08em
    T\kern-.1667em\lower.7ex\hbox{E}\kern-.125emX}}
    
\begin{document}

\title{ModShift: Model Privacy via Designed Shifts
\thanks{This work has been funded in part by one or more of the following grants: ONR N00014-22-1-2363, NSF CCF 2148313, ARO W911NF2410094, NSF CCF-2311653, ARO W911NF1910269, and NSF DBI-2412522.}

}

\author{\IEEEauthorblockN{Nomaan A. Kherani}
\IEEEauthorblockA{
\textit{University of Southern California}\\
kherani@usc.edu}
\and
\IEEEauthorblockN{Urbashi Mitra}
\IEEEauthorblockA{
\textit{University of Southern California}\\
ubli@usc.edu}}

\maketitle

\begin{abstract}
In this paper, shifts are introduced to preserve model privacy against an eavesdropper in federated learning. Model learning is treated as a parameter estimation problem. This perspective allows us to derive the Fisher Information matrix of the model updates from the shifted updates and drive them to singularity, thus posing a hard estimation problem for Eve. The shifts are securely shared with the central server to maintain model accuracy at the server and participating devices. A convergence test is proposed to detect if model updates have been tampered with and we show that our scheme passes this test. Numerical results show that our scheme achieves a higher model shift when compared to a noise injection scheme while requiring a lesser bandwidth secret channel. 
\end{abstract}


\section{Introduction}
In federated learning (see {\em e.g.} \cite{kairouz2021advances}), agents provide local model information to a global server while maintaining privacy of the local data resident at each agent. However, inferences about an agent's data can still be made exploiting the shared local model updates, motivating the development of additional strategies to ensure data privacy \cite{privacyreview,secagg,fastsecagg,lightsecagg,dp}. Schemes like secure aggregation \cite{secagg,fastsecagg,lightsecagg}, and differential privacy \cite{dp,tasnim2023approximating} address privacy of user data but do not address the privacy of the overall global model. 
%Federated learning has emerged as a promising methodology by which agents or sub-systems can engage in local inference about data which is to remain private and local, but still share model-relevant information that can be fused at a central server while still maintaining the privacy of the local data sets. 

Recent work has begun to address the issue of {\bf model privacy}\cite{modelprivacy1,modelprivacy2,modelprivacy3}. In \cite{modelprivacy1}, the model is protected from the participating agents, but not eavesdroppers. In contrast, \cite{modelprivacy2} protects the model from eavesdroppers, but does not enable agent model learning.  While \cite{modelprivacy3} protects model privacy from eavesdroppers, a very constrained scenario is considered: only a link between a single agent the global server (amongst many) is compromised. 
%In this case, algorithms in which agents share model increments are inherently protected.  
Herein, we consider the problem of model privacy when {\bf all} links between the agents and the global server are eavesdropped.

Our approach is inspired by the creation of statistically hard estimation problems for the eavesdropper through signal shaping \cite{da2024guaranteed,li2024channel,li2024channelb, li2024optimized}. In \cite{da2024guaranteed}, wireless communications are made private through randomizing the {\em structure} of the modulation.  In \cite{li2024channel, li2024channelb, li2024optimized}, localization is made private by modifying the channel perceived by the eavesdropper.  The Fisher Information Matrix of the estimation problem undertaken by the eavesdropper is driven to singularity through transmitted signal precoding.  Such an approach is undertaken herein for the new problem of model privacy in federated learning or distributed optimization.  As in \cite{da2024guaranteed,li2024channel,li2024channelb, li2024optimized}, only a modest amount of information is shared with the global server by the agents in order to achieve model privacy.

 \begin{figure}[t]
 \begin{center}
\includegraphics[width=0.4\textwidth]{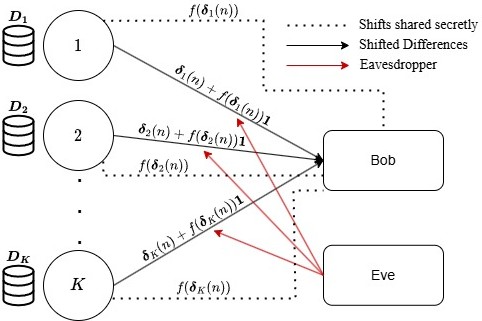}
\caption{System model for distributed optimization, model shift and eavesdropping.}
\label{fig:FL}
\end{center}
\end{figure}
We introduce intentional model shifts in the updates shared with the global server, effectively introducing model shifts into the perceived distributions of the data at each agent.  We use FedAvg \cite{mcmahan2017} as our exemplar learning scheme; however, our approach can be adapted to other federated or disributed optimization methods (see {\em e.g.} \cite{verma2023maximal}).
Model shift due to data heterogeneity has proven to be detrimental for the learning of global models \cite{shift1,shift2}.
The shifts are designed to pose a provably hard estimation problem for the eavesdropper and are shared with the central server over a secret channel. 
%Thus allowing the agents and central server to learn the global model while leading eavesdroppers to a shifted incorrect global model.

The main contributions of this paper are
\begin{enumerate}
    \item  The Fisher Information Matrix of the eavesdropper's (Eve's) model estimation problem is driven to singularity via model shift strategy wherein each agent secretly shares shift information with the central server (Bob) (see Figure ~\ref{fig:FL}).
    \item A family of shift designs is designed  that successfully passes a proposed test to detect if model updates have been tampered with.
    \item Several theoretical results are provided including the provable goodness of the method and a convergence analysis.
    \item The proposed method is compared to noise injection methods.  Noise injection fails the tamper test and requires a greater bandwidth secret channel and further results in a significantly lower estimation loss than the proposed scheme.
\end{enumerate}

\section{System Model}
 We consider a network consisting of $K$ agents that communicate with a central server (Bob) to collectively learn a global model $\boldsymbol{w}^{*} \in \mathbb{R}^{d}$. An unauthorized receiver (Eve) tries to learn $\myb{w}^{*}$ by eavesdropping on the communication between the agents and the server. Each agent has a local dataset $\myb{D}_{k}$ comprised of realizations of feature vectors, $\boldsymbol{x}_{k,i} \in \mathbb{R}^{d}$, and their corresponding labels $y_{k,i}$, where 
 $\myb{D}_{k} = \left [[\boldsymbol{x}_{k,1}^{T},y_{k,1}]^{T} , \cdots, [\boldsymbol{x}_{k,m_k}^{T},y_{k,m_{k}}]^{T} \right]$.  The global data set is thus,
$\myb{D} =  \left[\myb{D}_{1}, \cdots ,\myb{D}_{K}
\right] \in  \mathbb{R}^{(d+1) \times m},$ where $m = \sum_{k=1}^{K}m_{k}.$ 

 The function $l(\boldsymbol{w},\boldsymbol{x}_{k,i},y_{k,i})$ represents the loss of the $i$th data sample of agent $k$.
 Our aim is to find the $\boldsymbol{w}^{*}$ that minimizes the global loss function
 \begin{align}
    \label{loss}
     F(\boldsymbol{w}) = \frac{1}{m}\sum_{k=1}^{K}\sum_{i=1}^{m_{k}}l(\boldsymbol{w},\boldsymbol{x}_{k,i},y_{k,i}).
 \end{align}
 The gradient computed by agent $k$ on dataset $\myb{D}_{k}$ with a weight $\myb{w}$ is given by,
\begin{equation}
    \myb{g}_{k}(\myb{w},\myb{D}_{k}) = \sum_{i=1}^{m_{k}} \frac{\nabla l(\myb{w},\myb{x}_{k,i},y_{k,i})}{m_{k}}.    
\end{equation}
We use FedAvg \cite{mcmahan2017} with a slight modification to find the parameter $w^{*}$ as presented in Algorithm (\ref{algo}). In iteration $n$, after performing $R$ steps of local gradient descent with weight $\myb{w}(n)$, the devices share the \emph {difference} between their resulting local model $\myb{w}_{k,R}(n)$ and the global model $\myb{w}(n)$ instead of sharing the local model $\myb{w}_{k,R}(n)$. This formulation introduces an initial model shift due to random initialization and necessitates that Eve eavesdrop in every communication round to attempt to reconstruct the model trajectory. 
\begin{algorithm}[t]
\caption{The $K$ clients are indexed by $k$; $R$ is the number of local epochs and $\eta$ is the learning rate.}
\begin{algorithmic}[t]
\State \textbf{Server executes:}
\State Initialize $\myb{w}(0)$
\For{each round $n = 1, 2, \dots N$}
    \For{each client $k$ \textbf{in parallel}}
        \State $\myb{w}_{k,0}(n) \gets \myb{w}(n)$
        \For{each local epoch $r$ from $0$ to $R-1$}
            \State $\myb{w}_{k,r+1}(n) = \myb{w}_{k,r}(n) - \eta g_{k}(\myb{w}_{k,r}(n),\myb{D}_{k})$
        \EndFor
        \State $\myb{\delta}_{k}(n) \gets \myb{w}_{k,R}(n)-\myb{w}(n)$
    \EndFor
    \State $\myb{w}(n+1) \gets \myb{w}(n) + \sum_{k=1}^{K}\frac{m_{k}}{m}\myb{\delta}_{k}(n)$
\EndFor
\end{algorithmic}
\label{algo}
\end{algorithm}

We assume that each agent uses an orthogonal channel for transmission of $\myb{\delta}_{k}(n)$ to Bob, such as orthogonal frequency-division multiplexing (OFDM) \cite{ofdmbook}. Under the assumption of white Gaussian noise across the sub-channels and flat fading with known channel state information for both Bob and Eve, we have that the effective received signals per orthogonal channel in iteration $n$ for each receiver is
\begin{eqnarray}
    \label{Bob1}
    {\underline{\myb{y}}_{k}^{B}(n)} & = &\myb{\delta}_{k}(n) + \underline{\myb{z}}_{k}^{B}(n),\\
    \label{Eve1}
    {\underline{\myb{y}}_{k}^{E}(n)} & = &\myb{\delta}_{k}(n) + \underline{\myb{z}}_{k}^{E}(n),
\end{eqnarray}
where $\underline{\myb{z}}_{k}^{u}(n) \sim \mathcal{CN}(0,\frac{\sigma_{u}^{2}}{h_{k}^{u}(n)^{2}}\myb{I})$ is the distribution of the noise conditioned on the known channel state $h_{k}^{u}(n)$ where $u \in \{ B,E\}$. Both Bob and Eve wish to estimate the global model $w^{*}$ from their received signals.

%. Bob estimates $\myb{\delta}_{k}(n)$ from his received signal. Eve eavesdrops on the channels. The effective received signal per orthogonal channel in iteration $n$ at Eve ($\myb{y}_{k}^{B}(n)$) is
%where $\underline{\myb{z}}_{k}^{E}(n) \sim \mathcal{CN}(0,\frac{\sigma_{E}^{2}}{h_{k}^{E}(n)^{2}}\myb{I})$ is the distribution of the noise conditioned on the known channel state $h_{k}^{E}(n)$. Eve estimates $\myb{\delta}_{k}(n)$ from her received signals which is used to learn $w^{*}$.  

In the sequel, we modify the transmitted signals, sharing this change with Bob over a secure channel \footnote{Note that communication privacy can be achieved via the methods in \cite{da2024guaranteed,li2024channel,li2024channelb, li2024optimized}.}, to mislead Eve to the incorrect model.
\section{Proposed Scheme: ModShift}
%In this section, we present an outline of ModShift. Each agent, in each communication round, shifts their differences $\myb{\delta}_{k}(n)$ by an amount $\myb{c}_{k}(n)$ before transmission to prevent Eve from accurately estimating the differences. The shifts are shared with Bob over a secret channel by the agents to ensure they converge to the true model. 
%\um{what if Bob or Eve applies an SVD? How much is lost?}
We propose to introduce shifts to $\myb{\delta}_{k}(n)$ to degrade the performance of the algorithm at Eve. With the addition of the shift, Eve estimates $\myb{\delta}_{k}(n)$ from 
\begin{align}
    \label{Eve}
    \underline{\myb{y}}_{k}^{E}(n) &  = \myb{\delta}_{k}(n) + f_{k,n}(\myb{\delta}_{k}(n))\mathbbm{1} + \underline{\myb{z}}_{k}^{E}(n),   
\end{align}
where $f_{k,n}(\cdot):R^{d} \rightarrow R$ is a function which we design and $f_{k,n}(\myb{\delta_{k}}(n))\mathbbm{1}$ is the shift introduced by the agents.
Thus, Eve's update is 
\begin{align}
    \label{update}
    \myb{w}(n+1)^{E} = \myb{w}(n)^{E} + \sum_{k=1}^{K}\frac{m_{k}}{m}\underline{\myb{y}}_{k}^{E}(n)
\end{align}
In the sequel, we provide a design for $f_{k,n}(\cdot)$ that yields a challenging estimation problem for Eve.  
%The shifts are provided to Bob via a secret channel in order to ensure that the server can properly perform the federated learning algorithm.
After compensating for shifts, Bob's received signal remains as in Equation (\ref{Bob1}).
The goal is for Eve's estimate to converge to the incorrect model.  The difference between the true and estimated model is called the \emph{model shift}.  To  minimize the use of the secret channel, the shift is fixed for each feature dimension. 
%The shift is a function of the gradient to be transmitted and is given by $f_{k,n}(\myb{\delta}_{k}(n))\mathbbm{1}$.

\section{Model Shift Design}
We shall design our shifts in order to challenge Eve's model estimation ability.
To this end, we will develop designs that will drive Eve's Fisher Information matrix \cite{stevenestimation} to singularity. To show our results, we have the following assumptions:

\textbf{Assumption 1.} Given the data, $\myb{\delta}_{k}(n)$ is deterministic.

\textbf{Assumption 2.} The elements of $\myb{\delta}_{k}(n)$ are not functions of each other, \emph{i.e.} $\frac{\partial \myb{\delta}_{k,j}(n)}{\partial \myb{\delta}_{k,i}(n)} = 0 \hspace{5pt} \forall i \neq j$.

%We design shifts for two cases. In the first case, we assume that Eve does not know that the differences are being shifted by the agents and therefore only estimates $\myb{\delta}_{k}(n)$ in Equation (\ref{Eve}). In the second case, Eve is aware that the differences are being shifted and estimates $\myb{\delta}_{k}(n)$ as well as $\myb{c}_{k}(n)$ in Equation (\ref{Eve}). The signal received by Eve in both cases is given by Equation (\ref{Eve}).
 Eve attempts to estimate  $\myb{\delta}_{k}(n)$ , $ \forall \;\; k$ and each round $n$ from Equation (\ref{Eve}).  To describe the estimation capability for this scenario, we have a set of Fisher Information Matrices (FIMs) for each $k$ and $n$.
 The problem  to solve is
\begin{eqnarray}
    \mathcal{P}: f_{k,n}(\cdot) &  = & \arg \min_f \left|\mbox{det} \left(\myb{J}(\myb{\delta}_{k}(n))\right) \right |
\end{eqnarray}
where $\myb{J}(\myb{\delta}_{k}(n))$ is the Fisher Information Matrix of the differences sent by agent $k$ in round $n$ and is given by
\begin{align}
    \myb{J}(\myb{\delta}_{k}(n)) = -\mathbb{E} \left[\nabla_{\myb{\delta}_{k}(n)}^{2}\log p_{\underline{\myb{y}}_k^E(n)}(\myb{y}; \myb{\delta}_k(n))\right],
\end{align}  
where $p_{\underline{\myb{y}}_k^E(n)}(\cdot;\myb{\delta}_k(n))$ denotes the probability density function of $\underline{\myb{y}}_k^{E}(n)$, parameterized by $\myb{\delta}_k(n)$.
Our next task is to characterize the FIM for Eve.  We shall see that we can determine constraints on the shifts to ensure that her FIM is singular.
\begin{lemma}
\label{FIMg}
 $\myb{J}(\myb{\delta}_{k}(n))$ for the signal model in Equation (\ref{Eve}) is given by
 %\begin{eqnarray}
%\begin{dmath}
    %\myb{J}(\myb{\delta}_{k}(n)) = & \! \! \! \! \! \frac{\sigma_{E}^{2}}{h_{k}^{E}(n)^{2}}\left( \mathbf{I} + \mathbbm{1} \nabla f_{k,n}(\myb{\delta}_{k}(n))^T + \nabla f_{k,n}(\myb{\delta}_{k}(n))  \mathbbm{1}^T \right. \nonumber\\     + & \left. \! \! \! \! \! d \nabla f_{k,n}(\myb{\delta}_{k}(n)) \nabla f_{k,n}(\myb{\delta}_{k}(n))^T \right). \! \! \! \! \! 
%\end{dmath}
%\end{eqnarray}
\begin{eqnarray}
    & \! \! \! \! \! \! \myb{J}(\myb{\delta}_{k}(n)) =  \frac{2h_{k}^{E}(n)^{2}}{\sigma_{E}^{2}}\left( \mathbf{I} + \mathbbm{1} \nabla f_{k,n}(\myb{\delta}_{k}(n))^T + \right. \nonumber \\ & \! \! \! \! \! \! \left. \nabla f_{k,n}(\myb{\delta}_{k}(n))  \mathbbm{1}^T + d \nabla f_{k,n}(\myb{\delta}_{k}(n)) \nabla f_{k,n}(\myb{\delta}_{k}(n))^T \right).
\end{eqnarray}
where $\myb{\delta}_{k,i}(n)$ is the $i$th element of the vector $\myb{\delta}_{k}(n)$ and $\myb{e}_{i}$ is a standard basis vector in $\mathbb{R}^{d}$, where the $i$th component is one.
\end{lemma}
This result follows straightforwardly from
the FIM for estimation of a parametric mean in Gaussian noise \cite{stevenestimation}. To provide a constraint on $f_{k,n}(\cdot)$ such that the above matrix is singular, we exploit the Matrix Determinant Lemma \cite{matdet}: 
 \begin{lemma}
     (Matrix Determinant Lemma). Suppose $\myb{A}$ is an invertible square matrix of dimension $n$ and $\myb{U},\myb{V}$ are matrices with dimension $n \times m$. Then
     \begin{equation}
         \mbox{det}(\myb{A} + \myb{U}\myb{V}^{T}) = \mbox{det}(\myb{A})\mbox{det}(\myb{I}_{m} + \myb{V}^{T}\myb{A}^{-1}\myb{U}).
     \end{equation}
 \end{lemma}
\begin{prop}
\label{singular1}
    $\myb{J}(\myb{\delta}_{k}(n))$ is singular if
    \begin{equation}
        \label{singular}
       \nabla f_{k,n}(\myb{\delta}_{k}(n))^{T}\mathbbm{1} = -1.
    \end{equation}
  Furthermore, under this condition, the eigenvalues of $\myb{J}(\myb{\delta}_{k}(n))$ are given by
    \begin{align}
    \label{eig1}
        \lambda_{1}& = 0,   \; \;   \lambda_{j}  = \frac{2h_{k}^{E}(n)^{2}}{\sigma_{E}^{2}} \; \;  j = \{2, \cdots, d-1\}, \; \; \\
        &\lambda_{d} = \frac{2h_{k}^{E}(n)^{2}d}{\sigma_{E}^{2}}||\nabla f_{k,n}(\myb{\delta}_{k}(n))||^{2}. \nonumber
    \end{align}
\end{prop}
\begin{proof}
Let $\myb{J}_{1}(\myb{\delta}_{k}(n)) = \frac{\sigma_{E}^{2}}{2h_{k}^{E}(n)^{2}}\myb{J}(\myb{\delta}_{k}(n))$ and  $\lambda_{1},\cdots,\lambda_{d}$ be the eigenvalues of $\myb{J}(\myb{\delta}_{k}(n))$. We compute $\det(\myb{J}_{1}(\myb{\delta}_{k}(n)) - \lambda \myb{I})$ to find its eigenvalues. We observe from Proposition ~\ref{FIMg}, that $\myb{J}_{1}(\myb{\delta}_{k}(n)) - \lambda \myb{I}$ can be written as the sum of a $d \times d$ matrix diagonal matrix $\myb{A}$ and the product of two $d \times 3$ matrices $\myb{U}$ and $\myb{V}$ such that,
\begin{align}
    \label{eigenvalue}
    \myb{J}_{1}(\myb{\delta}_{k}(n)) - \lambda \myb{I} = \myb{A} + \myb{U}\myb{V}^{T},
\end{align}
where 
\begin{align}
    \myb{A} & =  \mbox{diag} \left[ ( 1- \lambda) \mathbbm{1} \right] \\ 
    \myb{U} & = \begin{bmatrix}
        \nabla f_{k,n}(\myb{\delta}_{k}(n))&  \mathbbm{1} & \sqrt{d}\nabla f_{k,n}(\myb{\delta}_{k}(n))
    \end{bmatrix},\\
    \myb{V} & = \begin{bmatrix}
    \mathbbm{1} & \nabla f_{k,n}(\myb{\delta}_{k}(n)) & \sqrt{d}\nabla f_{k,n}(\myb{\delta}_{k}(n))
    \end{bmatrix}.
\end{align}
We use the matrix determinant lemma to compute the determinant of Equation (\ref{eigenvalue}) and set it to $0$ to obtain the eigenvalues of $\myb{J}_{1}(\myb{\delta}_{k}(n))$ from which we obtain the eigenvalues of $\myb{J}(\myb{\delta}_{k}(n))$. We then set $\lambda_{1} = 0$, giving us the condition in Equation (\ref{singular}). 
\end{proof}
We observe that these conditions are sufficient, but not necessary, that is, other solutions may exist.
%\no{Proposition \ref{singular1} gives us a condition on the gradient of $f_{k,n}(\cdot)$ that drives Eve's Fisher Information Matrix $\myb{J}(\myb{\delta}_{k}(n))$ to singularity. Designing $f_{k,n}(\cdot)$ to satisfy Equation (\ref{singular}) makes it hard for Eve to estimate the differences.  }
Our next task is to determine the solution to the differential equations provided in Equation (\ref{singular}).
\begin{prop}
    A family of solutions to Equation (\ref{singular}) is given by
    \begin{dmath}
        \label{sol}
        f_{k,n}(\myb{\delta}_{k}(n)) = \myb{\gamma}_{k}(n)^{T}\myb{\delta}_{k}(n) + g\left(\myb{D}\myb{\delta}_{k}(n)\right), 
    \end{dmath}
    where $g(\cdot) : \mathbb{R}^{d-1} \rightarrow \mathbb{R}$ is an arbitrary differentiable function, 
    %\um{this is really the only constraint?  does not need to be bounded?} \no{This is the only constraint if we solve (\ref{singular}). A boundedness constraint may be needed when considering power but it seems that mathematically, differentiability of $g(\cdot)$ is the only constraint} 
    $\myb{\gamma}_{k}(n) \in \mathbb{R}^{d}$ such that
    $\myb{\gamma}_{k}(n)^{T}\mathbbm{1} = -1$ and
    \begin{align}
   \! \! \! \! \!      \myb{D}\myb{\delta}_{k}(n) = \begin{bmatrix}
           \myb{\delta}_{k,1}(n) - \myb{\delta}_{k,2}(n) & \! \! \! \! \!  \cdots  \! \! \! \! \! 
           & \myb{\delta}_{k,1}(n) - \myb{\delta}_{k,d}
         \end{bmatrix}^{T}
    \end{align}
\end{prop}
The method of characteristics \cite{characteristics} is used to obtain this result. The shift design consists of a linear term $\myb{\gamma}_{k}(n)^{T}\myb{\delta}_{k}(n)$ which satisfies $\nabla_{\myb{\delta}_{k}(n)} (\myb{\gamma}_{k}(n)^{T}\myb{\delta}_{k}(n))\mathbbm{1}^{T} = -1$ and a free term which satisfies $\nabla_{\myb{\delta}_{k}(n)} g(\myb{D} \myb{\delta}_{k}(n))^{T} \mathbbm{1}= 0
$. A valid function is the zero function $g(\cdot) = 0$ for all arguments (see Equation (\ref{sol})). {Investigating alternative $g(\cdot)$  functions is a task for future work.} In the case of $g(\cdot) = 0$, we can write the signal transmitted by agent $k$,
\begin{align}
 \! \! \! \! \!    \myb{\delta}_{k}(n) + f_{k,n}(\myb{\delta}_{k}(n))\mathbbm{1} %& = \myb{I}\myb{\delta}_{k}(n) + \myb{\gamma}_{k}(n)^{T}\myb{\delta}_{k}(n)\mathbbm{1},\\
   % & = \myb{I}\myb{\delta}_{k}(n) + \left(\mathbbm{1}\myb{\gamma}_{k}(n)^{T}\right)\myb{\delta}_{k}(n),\\
     = \left[I+\left(\mathbbm{1}\myb{\gamma}_{k}(n)^{T}\right)\right]\myb{\delta}_{k}(n). 
\end{align}
 We can show that $\left[I+\left(\mathbbm{1}\myb{\gamma}_{k}(n)^{T}\right)\right]$ has rank $d-1$.  Thus, a linear transformation to produce the desired shifts from the differences can drive Eve's FIM to singularity. 
\section{Convergence}
\label{convergence}
Under ideal conditions where $\sigma_{E},\sigma_{B} \rightarrow 0$, Eve expects $||\myb{w}(n+1)^{E}-\myb{w}(n)^{E}|| \rightarrow 0$ when Bob converges. Eve can use this as a test to investigate if gradients have been tampered with. We show that our scheme passes this test despite the addition of shifts. A similar result can be shown for non-zero $\sigma_{E},\sigma_{B}$.
\begin{prop}
    \label{convprop}
    Given $||\myb{w}(n+1) - \myb{w}(n)|| \leq \epsilon,$  Eve's weights are bounded as
    \begin{eqnarray}
      & & \! \! \! \! \!  \! \! \! \! \!   \! \! \! \! \!   ||\myb{w}(n+1)^{E} - \myb{w}(n)^{E}|| \leq \nonumber \\
        &&
        \begin{cases}
             \epsilon(1+\sqrt{d}||\gamma(n)||), \text{ if } \myb{\gamma}_{k}(n) = \myb{\gamma}(n) \forall k,\\
            \epsilon(1+\sqrt{d} \max_{k}||\gamma_{k}(n)||\alpha(n)), \text{else}
        \end{cases}
    \end{eqnarray}
    where $\alpha(n) = \frac{\sum_{k=1}^{K}\frac{m_{k}}{m}||\myb{\delta}_{k}(n)||}{\left|\left|\sum_{k=1}^{K}\frac{m_{k}}{m}\myb{\delta}_{k}(n)\right|\right|}$.
\end{prop}
\begin{proof}
    Given $||\myb{w}(n+1) - \myb{w}(n)||  < \epsilon$,
    \begin{align}
        \label{epsilon}
        \implies \left|\left|\sum_{k=1}^{K}\frac{m_{k}}{m}\myb{\delta}_{k}(n)\right|\right| & < \epsilon. 
    \end{align}
We show the proof for $\myb{\gamma}_{k}(n) = \myb{\gamma}(n) \forall k$. From Equation (\ref{update}),
    \begin{dmath}
        ||\myb{w}(n+1)^{E} - \myb{w}(n)^{E}|| = \left|\left|\sum_{k=1}^{K}\frac{\left(m_{k}[\myb{\delta}_{k}(n)+f_{k,n}(\myb{\delta}_{k}(n))\mathbbm{1}]\right)}{m}\right|\right|,\\ 
         \stackrel{(a)}{\leq} \epsilon + \left|\left|\left(\sum_{k=1}^{K}\frac{m_{k}}{m}\gamma_{k}^{T}(n)\myb{\delta}_{k}(n)\right)\mathbbm{1}\right|\right|, \\
         \stackrel{(b)}{=} \epsilon +  \sqrt{d}\left|\gamma^{T}(n)\sum_{k=1}^{K}\frac{m_{k}}{m}\myb{\delta}_{k}(n)\right|,\\
        \stackrel{(c)}{\leq} \epsilon + \sqrt{d}||\gamma(n)||\left|\left|\sum_{k=1}^{K}\frac{m_{k}}{m}\myb{\delta}_{k}(n)\right|\right|,\\
         \leq \epsilon(1 + \sqrt{d} ||\gamma(n)||),
    \end{dmath}
    where  (a) is due to the triangle inequality and Inequality (\ref{epsilon}) which is a bound on the norm of Bob's update, (b) from the simplifying of the norm of a scalar multiple of a vector and (c) from
    the Cauchy-Schwarz inequality; finally Inequality (\ref{epsilon}) is used.
    If $\gamma_{k}(n) \neq \gamma(n) \hspace{5pt} \forall {k}$, we first apply the triangle inequality in inequality (a), followed by the Cauchy-Schwarz inequality. 
    %        Therefore,
%         \begin{dmath}
%         ||\myb{w}(n+1)^{E} - \myb{w}(n)^{E}||  \leq \epsilon + \left|\left(\sum_{k=1}^{K}\frac{m_{k}}{m}\gamma_{k}^{T}(n)\myb{\delta}_{k}(n)\right)\right|\sqrt{d},\\
 %       \label{gamma}
 %        \leq \epsilon + \sqrt{d}\left|\gamma^{T}(n)\sum_{k=1}^{K}\frac{m_{k}}{m}\myb{\delta}_{k}(n)\right|. 
\end{proof}
Proposition (\ref{convprop}) implies that if all agents use the same shift scheme $\myb{\gamma}(n)$ in a given round and if $||\myb{\gamma}(n)||$ is bounded, then Eve converges if Bob converges since the norm of Eve's model updates are upper bounded by a bounded multiple of the norm of Bob's model update. We do not guarantee Eve's convergence if Bob does not converge. If the agents use different shift schemes whose norms are bounded, Eve's convergence depends on the alignment of the model updates, represented by $\alpha(n)$, which can be thought of as a measure of IID data. 

%We would like to maximize the model shift while maintaining privacy. Thus, we have the following optimization problem 
%\begin{align}
%\max_{\alpha_{1},\alpha_{2}...\alpha_{d}} & \Delta \myb{w}(n)\\
%\text{s.t. } \sum_{i=1}^{d}\alpha_{1} & = -1,\\
%\sum_{i=1}^{d}\left(\frac{1}{\alpha_{i}} + 1\right)^{2} & \geq \theta,
%\end{align}
%where $\theta$ can be thought of as a privacy parameter and $\alpha_{i} = \frac{\partial f_{k,n}(\myb{\delta_{k}^{n}})}{\partial \myb{\delta}_{k,i}}$.

%\section{Shift Invariant Problems}
%While we ensure that Bob and Eve converge to different weights, the shift along each dimension of the weight is equal. There are a class of problems which we are yet to completely identify whose performance is not affected by an equal shift across all dimensions of the weight. 
\section{Simulation Results}
\begin{figure}[t]
\begin{center}
\includegraphics[width=0.5\textwidth,height=3in]{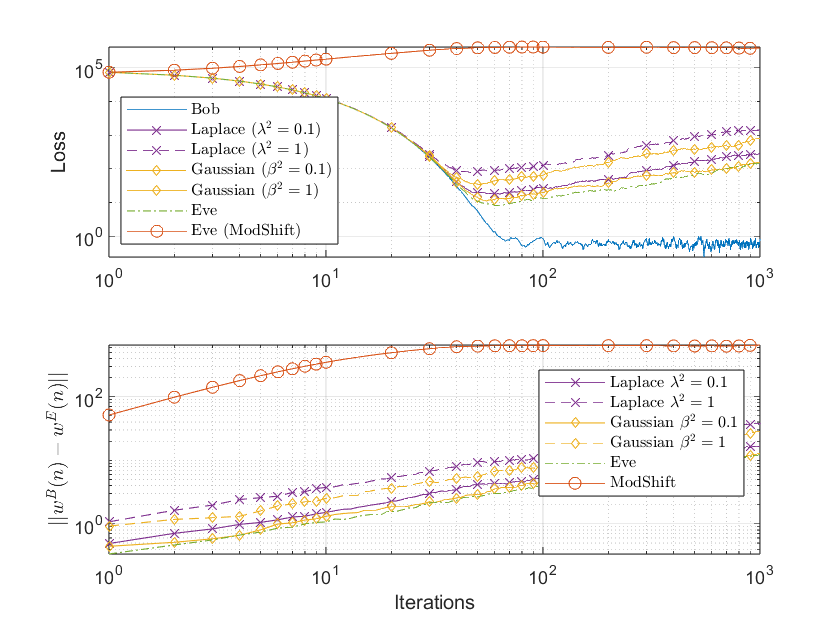}
\caption{Comparison of loss and model shift with ModShift and noise addition scheme}
\label{fig:comparison}
\end{center}
\end{figure}

We consider a linear regression problem on a synthetic dataset in a federated learning setup with $100$ independent agents communicating with Bob. The data vectors $\myb{x}_{k,i}$ are of dimension $d = 60$ and are Gaussian distributed with mean $0$ and identity covariance matrix. The weight vector $\myb{w} = [1,2, \cdots ,d]^{T}$ is fixed and labels for each data vector are generated as 
$y_{k,i} = \myb{w}^{T}\myb{x}_{k,i} + n_{k,i}$
where $n_{k,i}$ is a zero-mean Gaussian random variable with standard deviation $0.1$ for all $k,i$. Each agent has a dataset of size $1,000$. We use the mean square error loss function for which $l(\boldsymbol{w},\boldsymbol{x}_{k,i},y_{k,i})$ in Equation (\ref{loss}) is given by 
\begin{align}
    l(\boldsymbol{w},\boldsymbol{x}_{k,i},y_{k,i}) = (\myb{w}(n)^{T}\myb{x}_{k,i}-y_{k,i})^{2}.
\end{align}
We set $\frac{\sigma_{B}^{2}}{(h_{k}^{B}(n))^{2}} = \frac{\sigma_{E}^{2}}{(h_{k}^{E}(n))^{2}} = 0.1 \hspace{5pt} \forall k,n$ and $\myb{w}^{Bob}(0) = \myb{w}^{Eve}(0) = \myb{0}$ where $\myb{0} \in \mathbb{R}^{d}$ is the zero vector. The learning rate is $\eta=0.005$ and the number of local gradient rounds is $R = 10$. 

We compare ModShift to a noise injection scheme where each agent adds a noise vector to the differences $\myb{\delta}_{k}(n)$. We consider Gaussian noise with covariance matrix $\beta^{2}\myb{I}$ and Laplace noise with scale matrix $\lambda \myb{I}$. The $d$ elements of the noise vector are shared with Bob via the secret channel.

{We look at three shift schemes which satisfy the condition derived in Equation (\ref{sol})}: the Max, the Mean, and the Comp(onent) schemes. We observe that Max ModShift provides the largest shift and thus, not surprisingly, we will see has the maximum effect on Eve.\\ 
\noindent
%\um{could you have predicted which one would have done better?} \no{We expect the max scheme to lead to the highest model shift among the $3$ schemes since it shifts the updates by the highest amount and the data is IID.} \\
\textbf{Max scheme:} Let $\alpha = \arg \max_{i} |\myb{\delta}_{k,i}(n)|$. Then 
\begin{align*}
    \myb{\gamma}_{k,i}(n) = \begin{cases}
        -1,& i = \alpha,\\
        0, & i \neq \alpha.
    \end{cases}
\end{align*}\\
\textbf{Mean scheme:} $\myb{\gamma}_{k,i}(n) = -\frac{1}{d}, \hspace{5pt} \forall 1 \leq i \leq d.$\\
\textbf{Comp scheme:} $\myb{\gamma}_{k,i}(n) = \begin{bmatrix}
    -1 & 0 & \cdots & 0
\end{bmatrix}^{T}$.\\
Let $\Delta \myb{w}(n)$ denote the model shift.
Figure \ref{fig:comparison} shows the loss and $||\Delta \myb{w}(n)||$ as a function of the iterations for Bob and Eve under different schemes. For ModShift, we use the max scheme and for the noise injection scheme, we set $\beta^{2}, \lambda^{2} \in \{0.1,1\}$. Bob's performance is the same in both cases since the agents share the noise that is added with Bob. We also look at the model shift when the agents do not add noise or shifts (labelled Eve). ModShift achieves the highest $||\Delta \myb{w}(n)||$ and loss for Eve while the lowest $||\Delta \myb{w}(n)||$ is seen when agents do not shift or add noise to gradients. Increasing the parameters $\beta$ and $\lambda$ leads to an increase in the loss, as expected; however, this increase implies the sharing of more information over the secure channel. While for sufficiently large $\beta$ and $\lambda$, Eve's loss with the noise injection scheme will be worse than Eve's loss with ModShift, the drawback of the noise injection scheme is that each agent must share a $d$ dimension vector over the secret channel with Bob. Furthermore, the noise addition scheme does not pass the convergence test proposed in Section \ref{convergence}. 
\begin{figure}[t]
\begin{center}
\includegraphics[width=0.5\textwidth,height=3in]{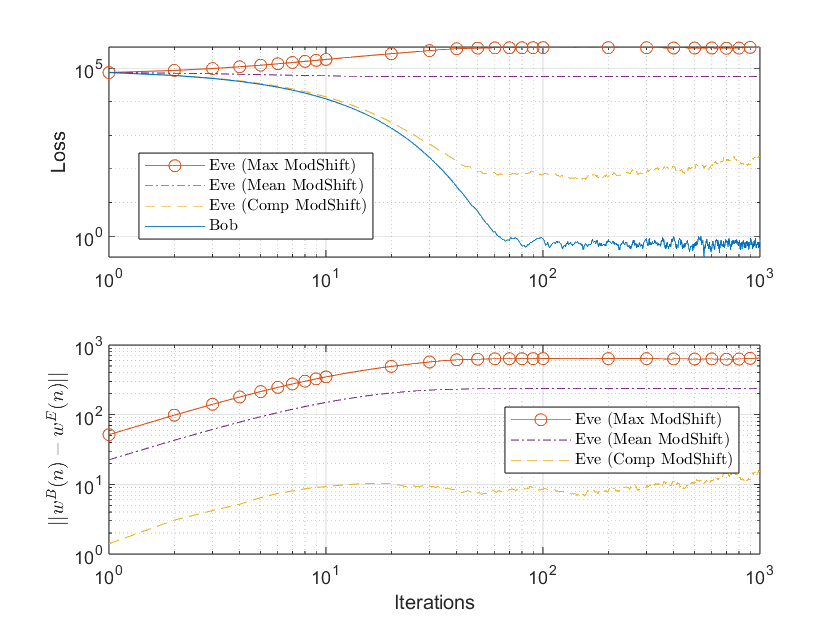}
\caption{Comparison of $3$ different ModShift schemes}
\label{fig:schemecomp}
\end{center}
\end{figure}
Figure \ref{fig:schemecomp} shows the convergence of the three different schemes. All three schemes lead to a model shift and a worse loss for Eve. The max scheme leads to the worst loss for Eve and also the largest $||\Delta \myb{w}(n)||$. The number of iterations after which Eve converges in the three schemes are similar.
\section{Conclusions}
We introduced a shift based approach to preserve model privacy against an eavesdropper in a federated learning environment. We derived the Fisher Information matrix for the shifted updates and theoretically identified a family of shift designs which drives the Fisher Information matrix to singularity. We also proposed a convergence test which Eve may use to identify tampered gradients and showed that our scheme passes this test. We observe that our approach achieves a higher model shift than a noise addition scheme while requiring lesser usage of the secret channel between Bob and each agent.  
%\begin{thebibliography}{00}
\bibliographystyle{IEEEtran}
\bibliography{References/references}
%\end{thebibliography}

\end{document}